\documentclass[11pt]{asaproc}

\usepackage{natbib}
\usepackage{graphicx}
\usepackage{amsmath, amssymb, amsthm}  
\usepackage{hyperref}  
\usepackage{subfig}  
\usepackage[symbol, perpage]{footmisc}  
\usepackage[many]{tcolorbox}  
\usepackage{relsize}  
\usepackage{mathtools}  
\usepackage{cprotect}
\usepackage[left=1in, right=1in, top=1in, bottom=1in]{geometry}
\usepackage[boxruled, algosection]{algorithm2e}  
\usepackage{chngcntr}  
\usepackage{abstract}  
\usepackage{array}
\usepackage{enumitem}  
\usepackage{algorithmic}

\usepackage{times}

\newtheorem{proposition}{Proposition}

\newtheorem{lemma}{Lemma}

\tcolorboxenvironment{proof}{
  blanker,
  before skip=\topsep,
  after skip=\topsep,
  borderline west={0.4pt}{0.4pt}{black},
  breakable,
  left=12pt,
  right=12pt, 
}

\newcommand{\bm}[1]{\boldsymbol{#1}}
\newcommand{\mb}[1]{\mathbb{#1}}
\newcommand{\mc}[1]{\mathcal{#1}}

\newcommand{\given}{\,\vert\,}

\newcommand{\st}{\,:\,}
\newcommand{\set}[1]{\left\{ #1 \right\}}

\newcommand{\real}{\mathbb{R}}

\newcommand{\trace}{\text{tr}}

\newcommand{\rank}{\text{rank}}

\newcommand{\sign}{\text{sign}}
\newcommand{\dnorm}[1]{\left\Vert #1 \right\Vert}
\newcommand{\norm}[1]{\Vert #1 \Vert}

\DeclareMathOperator*{\argmax}{\text{argmax}}
\DeclareMathOperator*{\minimize}{\text{minimize}}
\DeclareMathOperator*{\maximize}{\text{maximize}}

\renewcommand{\tilde}[1]{\widetilde{#1}}

\makeatletter
\DeclareFontFamily{U}{tipa}{}
\DeclareFontShape{U}{tipa}{m}{n}{<->tipa10}{}
\newcommand{\arc@char}{{\usefont{U}{tipa}{m}{n}\symbol{62}}}
\newcommand{\arc}[1]{\mathpalette\arc@arc{#1}}
\newcommand{\arc@arc}[2]{
	\sbox0{$\m@th#1#2$}
	\vbox{
		\hbox{\resizebox{\wd0}{\height}{\arc@char}}
		\nointerlineskip
		\box0
	}
}
\makeatother

\graphicspath{{../second/figures/}}  

\counterwithin{figure}{section}
\counterwithin{proposition}{section}
\counterwithin{lemma}{section}
\counterwithin{equation}{section}

\title{
	ON THE HUMAN-RECOGNIZABILITY PHENOMENON OF ADVERSARIALLY TRAINED DEEP IMAGE CLASSIFIERS
}

\cprotect\author{
\textbf{Jonathan Helland}\footnote{Corresponding author.} \qquad \textbf{Nathan VanHoudnos} \\
 Software Engineering Institute \\
Carnegie Mellon University \\
\{\verb|jwhelland|, \verb|nmvanhoudnos|\}\verb|@sei.cmu.edu|
}

\begin{document}

\pagestyle{plain}
\maketitle

\begin{abstract}
In this work, we investigate the phenomenon that robust image classifiers have human-recognizable
features – often referred to as interpretability – as revealed through the input gradients of their score
functions and their subsequent adversarial perturbations. In particular, we demonstrate that state-of-the-art methods for adversarial training incorporate two terms – one that orients the decision boundary via
minimizing the expected loss, and another that induces smoothness of the classifier’s decision surface
by penalizing the local Lipschitz constant. Through this demonstration, we provide a unified discussion
of gradient and Jacobian-based regularizers that have been used to encourage adversarial robustness in
prior works. Following this discussion, we give qualitative evidence that the coupling of smoothness and
orientation of the decision boundary is sufficient to induce the aforementioned human-recognizability
phenomenon.
\end{abstract}

\section{Introduction}

An adversarial example is often defined as ``an input to a ML model that is intentionally designed by an attacker to fool the model into producing an incorrect output'' \citep{goodfellow2017is_attacking}. 
\citet{tsipras_robustness_2019} observed that the adversarial examples for adversarially trained image classifiers were clearly human-recognizable as particular classes in the training data -- a phenomenon that does not generally occur for non-adversarially trained image classifiers.

\begin{figure}[!h]
    \centering
    \includegraphics[width=.75\textwidth, trim=0 0 0 .2em, clip]{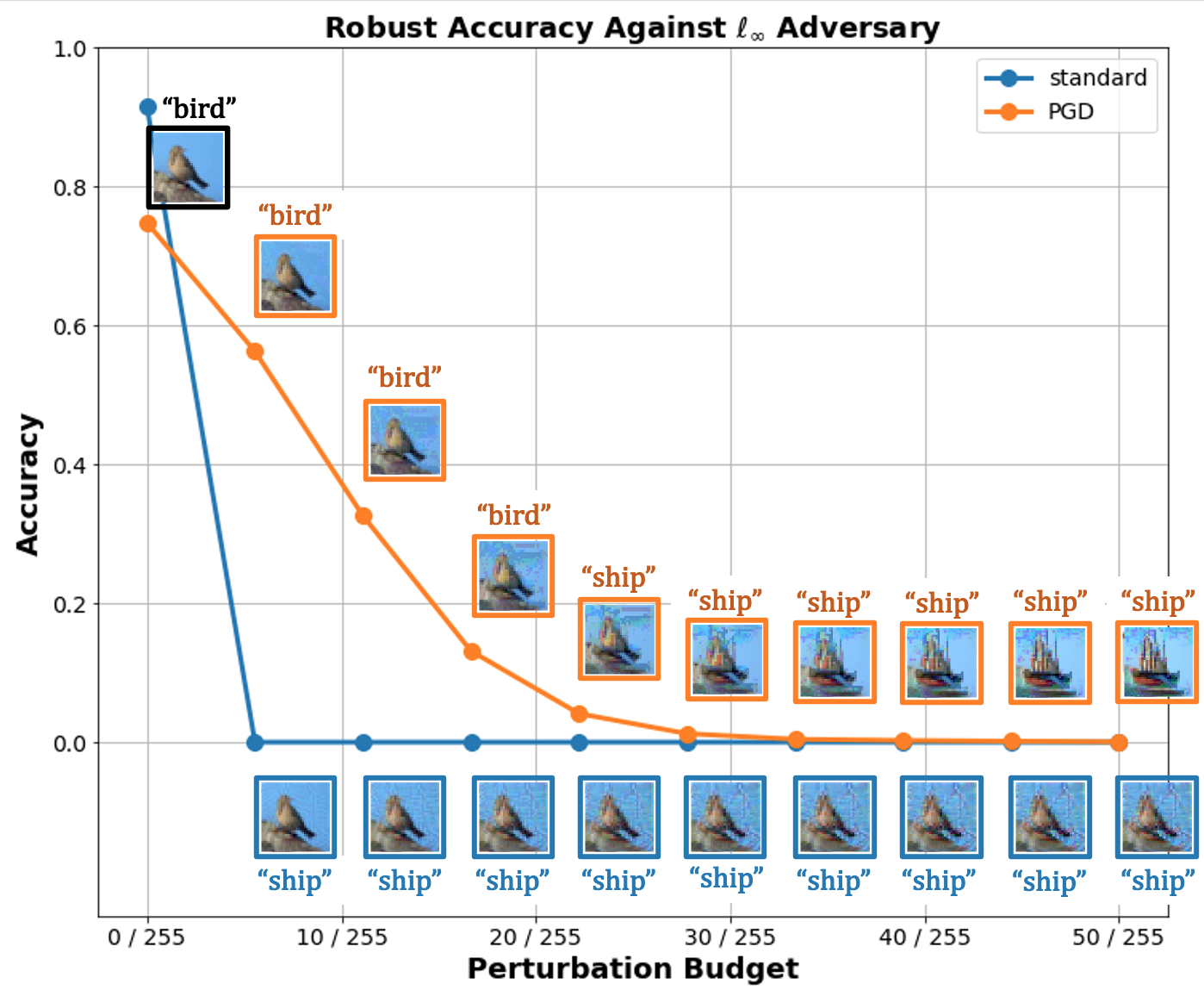}
    \caption{
        Illustration of the \textbf{human-recognizability} phenomenon.
        Adversarially (orange) and non-adversarially (blue) trained pre-activation ResNet-18 models \citep{he2016identity} trained on CIFAR-10 \citep{krizhevsky2009learning}. 
        The robust accuracy against an untargeted $\ell_\infty$ adversary are shown alongside targeted adversarial perturbations of a bird image from the test set using varying perturbation budgets.
    }
    \label{fig:motivation}
\end{figure}

Figure \ref{fig:motivation} illustrates this human-recognizability phenomenon by comparing the adversarial perturbations for two pre-activation ResNet18 models \citep{he2016identity} trained adversarially (orange, ``PGD'') and non-adversarially (blue, ``standard'') on the CIFAR-10 dataset \citep{krizhevsky2009learning}.
The images show targeted adversarial perturbations towards the \textit{ship} class of an image taken from the \textit{bird} class of the test-set, where each perturbation budget value indicates a separate perturbation starting from the same bird image (outlined in black).
The adversarially trained model produces perturbations that gradually become human-recognizable as the target class, whereas the non-adversarially trained model remains human-recognizable as a noisy version of the original class.


In this paper, we study adversarial examples from the perspective of human-recognizability.
We specifically ask the following question:

\begin{itemize}
\item[] \textit{
    What is a sufficient condition for the training method of a Convolutional Neural Network (CNN) image classifier such that adversarial perturbations against that CNN model are recognizable to humans?
}
\end{itemize}

\noindent
In answering this question, we make two contributions:

\begin{enumerate}
    \item A unified discussion of adversarial training methods around the idea of smoothness regularization. 
    This includes a generalization of Jacobian-based regularizers to $\ell_p$ with $p \geq 2$ adversaries and a theoretical discussion of the smoothness encouraged by this regularizer via the local Lipschitz constant of the classifier.
    
    \item An empirical demonstration that smoothness regularization is sufficient for human-recognizable adversarial examples.
    Here we provide qualitative evidence that smoothness regularizers promoting lesser degrees of adversarial robustness provide less human-recognizable adversarial perturbations, suggesting a potential trade-off between robustness and human-recognizability.
\end{enumerate}

Beyond demonstrating that the adversarial examples from smoothness-regularized models are human-recognizable, we do not consider the problem of data privacy in this paper that was suggested by \citet{mejia_robust_2019}. 
Rather, we take as a given that if human-recognizable patterns can be generated from a model without access to the training data, then there will exist privacy attacks capable of violating some security policy for a practical system. 

We conclude the paper with a brief discussion of the implications of the suggested trade-off between models that are robust to $\ell_p$ adversaries and the human-recognizability of their adversarial perturbations. 
Namely, that machine learning applications that require both robustness and data privacy require careful thought and further research.

\section{Background \& Related Work}\label{sec:background}

    \subsection{Background}
    We restrict our attention to models that perform image classification tasks, the setting in which adversarial examples were first observed \citep{szegedy2013intriguing}.
    We now discuss both attacks (generation of adversarial examples) and defenses (resistance to adversarial examples) that we will use in the rest of this work.

        \subsubsection{Attacks}
        Adversarial examples in the sense of untargeted and targeted evasion attacks are based on the following optimization problems.
        Take a classifier $f \st \mc{X} \to (0,1)^K$ over $K$-many classes with data sample $(\bm{x},y)~\sim~\mc{D}$ from the data generating distribution $\mc{D}$.
        We assume that the scores sum to one: $\sum_y f(\bm{x})[y] = 1$ for any input $\bm{x}$ (e.g. softmax outputs).
        For an untargeted attack, our goal is to find a solution to the constrained problem
        \begin{equation}\label{eq:untargeted attack}
            \max_{\bm{\delta} \in B_\epsilon} \mc{L}(\bm{e}_y, f(\bm{x} + \bm{\delta})),
        \end{equation}
        where $B_\epsilon$ is a constraint set (often an $\epsilon$-radius $\ell_p$-ball) and $\mc{L}$ is the loss function that consumes the one-hot label vector $\bm{e}_y \in [0,1]^K$.
        In words, find a perturbed image that changes the model's prediction to any other class.
        A targeted attack takes some other label $\tilde{y} \neq y$ and instead solves the problem
        \begin{equation}\label{eq:targeted attack}
            \min_{\bm{\delta} \in B_\epsilon} \, \mc{L}(\bm{e}_{\tilde{y}}, f(\bm{x} + \bm{\delta})),
        \end{equation}
        which finds a perturbed image that changes the model's prediction to a specific class. 
        
        The Madry Projected Gradient Descent (PGD) algorithm \citep{madry2017towards} is one of the most widely used methods for generating targeted and untargeted adversarial examples, and is based on the following iterative procedure.
        For some initial point $\bm{x}_0$ and some label $y$, it applies PGD to the above optimization problems:
    	\begin{equation}\label{eq:madry pgd iterates}
    		\bm{x}_{t+1} = \mathcal{P}\left( \bm{x}_t + \alpha \, \underset{ \norm{\bm{u}} \leq 1 }{ \argmax } \ \left\langle \bm{u}, \, \nabla_{\bm{x}} \mc{L} \left( \bm{e}_y, \, f_{\bm{\theta}}(\bm{x}_t) \right) \right\rangle \right)
    	\end{equation}
        for an arbitrary norm $\norm{\cdot}$ (often an $\ell_p$-norm, but works like \citep{wong2019wasserstein} consider other norms) and stepsize $\alpha$. 
        For an untargeted attack, we take $\alpha > 0$ and $y$ as the model's prediction for $x$, and for a targeted attack, we take $\alpha < 0$ with some particular $y$ chosen \textit{a priori}.
        The operator $\mathcal{P}$ projects iterates onto the constraint set $B_\epsilon = B_\epsilon(\bm{x}_0) \cap P$, where $B_\epsilon(\bm{x}_0)$ is canonically a radius-$\epsilon$ norm-ball centered at $\bm{x}_0$ and the set $P \equiv \set{ \bm{x} \st \bm{x}[i] \in [a,b] }$ for $a < b$, $a, b \in \real$ is the set of images with valid pixel values.
        

        \subsubsection{Defenses}
        We consider an undefended model as one whose learning task is only concerned with maximizing accuracy -- specifically, the optimization problem that minimizes the expected loss over the training data without any further regularization:
    	\begin{equation}\label{eq:undefended training}
			\minimize_{f \in \mc{F}} \ \mb{E}_{(\bm{x}, y) \sim \mc{D}} \left[  \mc{L}(\bm{e}_y, f(\bm{x})) \right].
		\end{equation}
        Here, the classifier $f$ belongs to a family of functions $\mc{F}$ (e.g. neural networks of a fixed architecture).
        It is well known that on common image classification tasks, neural networks trained via algorithms based on Eq. (\ref{eq:undefended training}) are highly susceptible to attacks like Eq. (\ref{eq:untargeted attack}) and Eq. (\ref{eq:targeted attack}) (see the blue curve in Figure \ref{fig:motivation} for illustration of this).
        
        The first defense against adversarial examples that we consider is Madry adversarial training \citep{madry2017towards}, which is based on the canonical adversarial learning program \citep{goodfellow2014explaining}
        \begin{equation}\label{eq:canonical adversarial training}
        	\minimize_{f \in \mc{F}} \ \mb{E}_{(\bm{x}, y) \sim \mc{D}} \left[ \max_{\bm{\delta} \in B_\epsilon} \, \mc{L}(\bm{e}_y, f(\bm{x} + \bm{\delta})) \right],
        \end{equation}
        where $B_\epsilon$ is canonically a radius-$\epsilon$ norm-ball, constraining the adversarial perturbation $\bm{\delta}$ to be no larger than $\epsilon$. 
        This program directly incorporates the untargeted attack Eq. (\ref{eq:untargeted attack}) in the training process as an adversary.
        Note that this minimax procedure is defined across the whole of the training data; the model never is exposed to unperturbed examples. 
        Madry adversarial training remains among the many proposed defenses one of the few that offers meaningful resilience against strong adversaries \citep{tramer2020adaptive, athalye2018obfuscated}. 
        
        It should be noted that the Fast Gradient Sign Method (FGSM) \citep{goodfellow2014explaining} can be seen as a single-iteration version of Eq. (\ref{eq:madry pgd iterates}) in which the projection operator $\mc{P}$ is removed.
        FGSM can be applied as an effective defense with some minor algorithmic improvements \citep{wong2020fast}, which is desirable for large datasets for which Madry adversarial training does not scale well.
        
        The second defense we consider is the TRADES algorithm proposed by \citep{zhang2019theoretically}, which also belongs to the camp of effective defenses.  
        It is based on a similar learning problem, albeit decomposed into an accuracy promoting term -- which is calculated on unperturbed data in contrast to Eq. \ref{eq:canonical adversarial training}) -- and a robustness promoting term which is calculated on perturbed data: 
        \begin{equation}\label{eq:trades}
        	\minimize_{f \in \mc{F}} \ \mb{E}_{(\bm{x},y) \sim \mc{D}} \left[ \mc{L}(\bm{e}_y, \, f(\bm{x})) + \max_{ \bm{\delta} \in B_\epsilon } \beta \, \mc{L}(f(\bm{x}), f(\bm{x} + \bm{\delta})) \right],
        \end{equation}
        where the regularization weight $\beta > 0$ trades off between the two terms. It was noted by \citep{zhang2019theoretically} that in the binary classification setting, $f$ tends to Bayes optimality as $\beta \rightarrow 0$, whereas $f$ tends towards an all-ones classifier as $\beta \rightarrow \infty$.
        
        Note that for TRADES, the attack is similar to Eq. (\ref{eq:untargeted attack}, only with $\nabla_{\bm{x}_t} \mc{L}( f(\tilde{\bm{x}}_0), f(\bm{x}_t) )$ instead of the loss with respect to the label $y$, where $\tilde{\bm{x}}_0 = \bm{x}_0 + \bm{\xi}$ with $\bm{\xi}$ random noise to prevent the gradient from vanishing initially.
        
        We also consider defenses closely related to canonical adversarial training Eq. (\ref{eq:canonical adversarial training}) and TRADES Eq. (\ref{eq:trades}) insofar as being approximations of both programs.
        In particular, we will derive connections between these programs and various other proposed defenses \citep{hoffman2019robust, moosavi2019robustness, jakubovitz2018improving, miyato2018virtual, zhao2019adversarial} via Taylor series approximations in Section \ref{sec:unification}.
    
        We also consider a historical defense, model distillation \citep{hinton2015distilling}, which smooths a model by raising its softmax temperature parameter and then training a new model using the smoothed predictions as pseudo-labels. 
        More specifically, distillation takes a pre-trained model with logit output $g \st \mc{X} \to \real^K$, define $p(y \given \bm{x}) \equiv \text{Softmax}( g(\bm{x}) / \tau )$ for some $\tau > 0$, and solve the learning problem
		\[
			\minimize_{f \in \mc{F}} \, \mb{E}_{(\bm{x}, y) \sim \mc{D}} \left[ \mc{L} \left( p(y \given \bm{x}), \, f(\bm{x}) \right) \right].
		\]
		Larger values of temperature $\tau$ push the scores of $g$ closer to a uniform distribution, which results in a smoother model $f$, which is called the distilled model.
		Defensive distillation \citep{papernot2016distillation} works by training the initial -- also called the teacher -- model at a raised temperature $\tau$ and then distills the model using the same temperature.
		Defensive distillation as a defense technique is well-known to be ineffective against adversarial examples \citep{carlini2016defensive}, thus we do not consider it to be a defense in the same sense as the others discussed in this section.
		However, defensive distillation provides a useful point of comparison with other methods because it is still a smoothness regularizer, which we will discuss in more detail in Section \ref{sec:smoothness via adversarial training}.

    \subsection{Related Work}
    Gradient-based regularizers for the promotion of adversarial robustness have been considered by \citep{ross2017improving, etmann2019connection, lin2019wasserstein, finlay2019scaleable}.
    Similarly, Jacobian-based regularizers have been investigated for similar purposes in \citep{hoffman2019robust, zhao2019adversarial, moosavi2019robustness}.
    
    Prior work has observed that the saliency maps (gradients of the score functions with respect to the input) of defended models evaluated at train and test images resemble the input points themselves -- a phenomenon often referred to as interpretability \citep{tsipras_robustness_2019, finlay2019scaleable, ross2017improving, tramer_adaptive_2020, etmann2019connection}.
    We use the term \textit{human-recognizability} instead of interpretability (which has been used in prior literature to describe this same phenomenon \citep{ross2017improving}) because interpretability connotates the idea of models that explain their decision processes transparently to the user; it is not clear to us that this visualization of model gradients has any meaningful correspondence with the traditional notion of interpretability. 
    Moreover, ``the term \textit{interpretability} holds no agreed upon meaning, and yet machine learning conferences frequently publish papers which wield the term in a quasi-mathematical way'' \citep{lipton2018mythos}.
    The term \textit{human-recognizability} better emphasizes the reliance on the qualitative nature of human perception and interpretation, which is a prevalent measure of quality in generative modeling literature (e.g. Generative Adversarial Networks (GANs)).

    It is observed in \citet{anil2019sorting} that classifiers whose Lipschitz constants are constrained have a similar gradient human-recognizability phenomenon.
    In \citet{kaur2019perceptually}, it is demonstrated that this human-recognizability phenomenon also occurs for randomized smoothing \citep{cohen2019certified}, and they further suggest that human-recognizability may be a general property of robust models.
    However, \citet{kaur2019perceptually} does not explore this latter claim for defenses beyond randomized smoothing.
    
    Works like \citet{tsipras_robustness_2019, ross2017improving, grathwohl2019your} have further noted that large adversarial perturbations (under an iterative procedure like Eq. (\ref{eq:madry pgd iterates} with large $\epsilon$) of input points (even noise) to a robust model remain recognizable.
    The hypothesis of \citet{ilyas2019adversarial} is that standard training Eq. (\ref{eq:undefended training}) is ill-posed and that adversarial training helps to reduce the set of features that the classifier can learn to compute to those that are more human-recognizable.
    In \citet{santurkar2019image}, this human-recognizability phenomenon of robust image classifiers is leveraged to perform computer vision tasks like inpainting and superresolution with a classifier instead of the usual generative modeling framework.
    Similarly, \citet{engstrom2019adversarial} demonstrates that robust image classifiers can be used for smooth image feature manipulation, also arguing that the feature representations of such models are approximately invertible i.e. the original training images can be approximately recovered.
    
    Model inversion \citep{fredrikson2015model} -- proposed as a privacy attack -- can be thought of as an unconstrained, targeted adversarial perturbation (see Appendix \ref{appendix:sec:model inversion}).
    It was observed in \citet{fredrikson2015model} that these targeted adversarial perturbations are recognizable for linear models and neural networks trained on a small dataset.
    Although model inversion is generally unrecognizable for undefended convolutional models trained on datasets with sufficiently diverse classes \citep{papernot2018sok, shokri2017membership}, it has been observed that model inversion applied to defended models yields highly recognizable perturbations \citep{terzi2020adversarial, mejia_robust_2019}.
    Although \citet{mejia_robust_2019} identifies the recognizably phenomenon as a privacy concern, the paper does not consider the conditions under which it appears beyond simply noting that it is present for popular defenses.
    
    In this work, we aim to build on \citet{mejia_robust_2019}, \citet{tsipras_robustness_2019}, and \citet{kaur2019perceptually} by identifying a sufficient condition for the human-recognizability phenomenon to occur.

\section{Unification of Adversarial Training around Smoothness Regularization}\label{sec:unification}
In this section, we argue that adversarial training of a classifier in the sense of the canonical program Eq. (\ref{eq:canonical adversarial training}) and the TRADES program Eq. (\ref{eq:trades}) both decompose via Taylor expansions into two terms, one that orients the decision boundary and one that smooths the decision surface.

First, we discuss undefended training Eq. (\ref{eq:undefended training}) as purely orientation of the decision boundary in Section \ref{sec:undefended training: orientation of the decision boundary}.
We then consider canonical adversarial training Eq. (\ref{eq:canonical adversarial training}), and show that a first-order Taylor approximation includes both orientation and smoothing.
Finally, we consider the TRADES program Eq. (\ref{eq:trades}), showing how this program includes decision boundary orientation as well as a stronger form of smoothing than canonical adversarial training.

Refer to Appendix \ref{appendix:sec:notation} for unfamiliar notation that is not defined explicitly within the text.

    \subsection{Undefended Training}\label{sec:undefended training: orientation of the decision boundary}
    In Eq. (\ref{eq:undefended training}), we only seek to minimize the expected loss $\mc{L}(\bm{e}_y, f(\bm{x}))$, which corresponds to maximizing the accuracy on the training set.
    This will happen when the model perfectly separates the classes into distinct regions separated by the decision boundary itself -- there is no preference for one geometry of the decision boundary over another beyond this separation.
    Traditional optimization wisdom then dictates that regularization terms should be used to express a more specific preference.
    Note that when solving the standard learning problem Eq. (\ref{eq:undefended training}) in practice, various kinds of regularization terms are typically introduced (e.g. weight decay, data augmentation, batch normalization, etc.).
    Moreover, the solution to the learning problem is computationally intractable, thus other factors such as model architecture, weight initialization, and the choice of optimization algorithm all influence the decision boundary of the trained model.
    

	\subsection{Adversarial Training: Smoothness \& Orientation}\label{sec:smoothness via adversarial training}
	Adversarial training can be used as an inductive bias towards decision boundaries that correspond to human-recognizability. 
	As an intuitive illustration of this, \citet{ilyas2019adversarial} shows that for a linear, binary classifier and two Gaussian distributed classes that are centered symmetrically (i.e. their means satisfy $\bm{\mu}_1 = -\bm{\mu}_2$), adversarial training prefers a model whose decision boundary is approximately orthogonal to the means (notice that for a linear classifier, the adversarial direction $\nabla_{\bm{x}} \mc{L}(\bm{e}_y, f(\bm{x}))$ will be orthogonal to the decision boundary).
	
	As such, we want to better understand the kind of regularization that adversarial training provides.
	Towards this understanding, we analyze local approximations of two adversarial learning problems: TRADES \citep{zhang2019theoretically} and canonical adversarial training \citep{madry2017towards, goodfellow2014explaining}.
	We identify smoothness regularization as the key addition in both cases.
	We select these two variants of adversarial learning because they form the foundation of a wide range of defenses against adversarial examples.

		\subsubsection{Canonical Adversarial Training}\label{sec:canonical adversarial training}
		We begin with canonical adversarial training \citep{madry2017towards, goodfellow2014explaining}, written in Eq. (\ref{eq:canonical adversarial training}).
		A first-order Taylor expansion of the loss term gives
		\begin{equation*}
			\mc{L}(\bm{e}_y, f(\bm{x} + \bm{\delta})) = \mc{L}(\bm{e}_y, f(\bm{x})) + \bm{\delta}^\top \nabla_{\bm{x}} \mc{L}(\bm{e}_y, f(\bm{x})) + \mc{O}(\norm{\bm{\delta}}_2^2).
		\end{equation*}
		For an $\ell_p$ threat model with $\mc{L}$ the cross-entropy loss, the solution $\bm{\delta}^\star$ to the inner maximization can be found by solving
		\begin{equation*}
			\max_{ \norm{\bm{\delta}}_p \leq \epsilon } \, \bm{\delta}^\top \nabla_{\bm{x}} \mc{L}(\bm{e}_y, f(\bm{x})) =  \epsilon \dnorm{ \nabla_{\bm{x}} \mc{L}(\bm{e}_y, f(\bm{x})) }_q = \frac{\epsilon}{ f(\bm{x})[y] } \dnorm{ \nabla_{\bm{x}} f(\bm{x})[y] }_q,
		\end{equation*}
		which is by definition the dual norm with $\frac{1}{p} + \frac{1}{q} = 1$.
		The solution is $\bm{\delta}^\star = \epsilon \, \norm{ \varphi_{q-1}( \nabla_{\bm{x}} \mc{L} ) }_q^{-1} \varphi_{q-1}( \nabla_{\bm{x}} \mc{L} )$ where $\varphi_q(\bm{x}) \equiv \sign(\bm{x}) \circ |\bm{x}|^q$ (see Appendix \ref{appendix:deriving gradient penalty} for derivation), which for an $\ell_\infty$ adversary reduces to $\bm{\delta}^\star = \epsilon \, \sign( \nabla_{\bm{x}} \mc{L} )$ whence the FGSM method as derived by \citep{goodfellow2014explaining}.
		In other words, up to a first-order approximation, gradient penalties on the loss are simply the FGSM algorithm.
		
		Plugging this approximation into Eq. (\ref{eq:canonical adversarial training}) yields the input-gradient penalized learning program
		\begin{equation}\label{eq:gradient penalty training}
			\minimize_{f \in \mc{F}} \, \mb{E}_{(\bm{x},y) \sim \mc{D}} \left[ \mc{L}(\bm{e}_y, f(\bm{x})) + \beta \epsilon \dnorm{ \nabla_{\bm{x}} \mc{L}(\bm{e}_y, f(\bm{x})) }_q \right],
		\end{equation}
		where $\beta > 0$ is an additional regularization weight.
		Eq. (\ref{eq:gradient penalty training}) with $q = 2$ is used by \citet{etmann2019connection, finlay2019scaleable, ross2017improving}.
		The classic norm-inequalities $\norm{ \nabla_{\bm{x}} f(\bm{x}) }_\infty \leq \norm{ \nabla_{\bm{x}} f(\bm{x}) }_2 \leq \norm{ \nabla_{\bm{x}} f(\bm{x}) }_1$ indicate that the $\ell_\infty$ adversary (whose dual norm is $\norm{\cdot}_1$) has the strongest smoothing effect.
		Moreover, the $1 / f(\bm{x})[y]$ term will discourage scores that are too large and too small, since we have the requirement that $\sum_y f(\bm{x})[y] = 1$ (this also happens with the TRADES penalty, see Appendix \ref{appendix:sec:proof of proposition regional smoothness}).	
		
		In the form of Eq. (\ref{eq:gradient penalty training}), we can see that within sufficiently small regions around the training data (i.e. sufficiently small $\epsilon$), the canonical adversarial training program Eq. (\ref{eq:canonical adversarial training}) decomposes into two terms: $\mc{L}(\bm{e}_y, f(\bm{x}))$ which orients the decision boundary and $\norm{\nabla_{\bm{x}} \mc{L}}_q$ which smooths the decision surface by penalizing the gradient and thereby the local Lipschitz constant around the data -- this idea is encapsulated in Proposition \ref{proposition:regional smoothness}, which we expound on in the following section.
		
		This decomposition is the foundation of our hypothesis that two key ingredients for human-recognizability are smoothing of the model and orienting the decision boundary.
		This hypothesis makes intuitive sense, since the orientation of the decision boundary is what aligns gradients towards high density regions of the data distribution and smoothness of the model helps prevent first-order iterative procedures like Eq. (\ref{eq:madry pgd iterates}) from getting stuck in local minima away from the data.

	\subsubsection{TRADES}\label{sec:trades}
	    We now consider the TRADES learning problem from Eq. (\ref{eq:trades}), and show that it yields a similar decomposition into a smoothing term and an orientation of the decision boundary term.
	    This culminates in Proposition \ref{proposition:regional smoothness}, which characterizes the impact of (a second-order Taylor approximation of) TRADES on the local Lipschitz constant of the model.
	    We also provide Lemma \ref{lemma:fisher information eigenvalue bounds}, which directly relates TRADES to other Jacobian-based smoothness regularizers.
	    
	    Recall the TRADES loss from Eq. (\ref{eq:trades}), $\mc{L}(\bm{e}_y, f(\bm{x})) + \max_{\bm{\delta} \in B_\epsilon} \, \beta \, \mc{L}(f(\bm{x}), f(\bm{x} + \bm{\delta}))$.
	    It is immediately clear that the loss term $\mc{L}(\bm{e}_y, f(\bm{x}))$ serves as the orientation term.
	    We now discuss the maximization term, which intuitively is indeed a smoothing regularizer as claimed by \citep{zhang2019theoretically}.
	    We articulate this smoothing more formally.
	
		
		Assuming an $\ell_p$ threat model with $p \geq 2$ and that $\mc{L}$ is the cross-entropy loss, we can connect the TRADES regularizer to more canonical Jacobian regularization via the Fisher Information Matrix (FIM).
		In particular, for the cross-entropy loss, we have an identify $\mc{L}(\bm{p}, \bm{q}) = H(\bm{p}) + \text{KL}(\bm{p} \,||\, \bm{q})$ in terms of Shannon entropy $H$ and the KL-divergence.
		Minimizing the entropy regularization term $H(f(\bm{x}))$ pushes $f$ towards confident class scores, which is redundant with the standard loss term $\mc{L}(\bm{e}_y, f(\bm{x}))$, so we drop this term from our approximation\footnote{
			In the implementation of TRADES provided by the authors, the inner maximization is with respect to the KL-divergence.		
			See \url{https://github.com/yaodongyu/TRADES/blob/master/trades.py\#L17-L84} for details.
		}.
		This means that the inner maximization reduces to $\max_{\bm{\delta} \in B_\epsilon} \, \beta \, \text{KL}(f(\bm{x}) \,||\, f(\bm{x} + \bm{\delta}))$.
		A second-order Taylor expansion yields
		\begin{equation}\label{eq:kl fisher approximation}
			\text{KL}(f(\bm{x}) \,||\, f(\bm{x} + \bm{\delta})) = \frac{1}{2} \bm{\delta}^\top \bm{F}_{\bm{x}} \bm{\delta} + \mc{O}( \dnorm{\bm{\delta}}_2^3 ),
		\end{equation}
		where 
		$\bm{F}_{\bm{x}} \equiv \mb{E}_y [ {\mathlarger{(}} \nabla_{\bm{x}} \mc{L}(\bm{e}_y, f(\bm{x})) {\mathlarger{)}} \, {\mathlarger{(}} \nabla_{\bm{x}} \mc{L}(\bm{e}_y, f(\bm{x})) {\mathlarger{)}}^\top ]$
		is the Fisher Information Matrix (FIM).
		This FIM can be rewritten in terms of the Jacobian of $f$ and another FIM: $\bm{F}_{\bm{x}} = \bm{J}(\bm{x})^\top \bm{F}_f \bm{J}(\bm{x})$, where 
		\[
			\bm{F}_f \equiv \mb{E}_y [ {\mathlarger{(}} \nabla_{f(\bm{x})} \mc{L}(\bm{e}_y, f(\bm{x})) {\mathlarger{)}} \, {\mathlarger{(}} \nabla_{f(\bm{x})} \mc{L}(\bm{e}_y, f(\bm{x})) {\mathlarger{)}}^\top ].
		\]
		Supposing the $\ell_p$ threat model $\bm{\delta} \in B_\epsilon = \set{ \bm{x} \in \mc{X} \st \norm{ \bm{x} }_p \leq \epsilon }$ and that the threshold $\epsilon > 0$ is sufficiently small relative to the smoothness of $f$, then the higher order terms in Eq. (\ref{eq:kl fisher approximation}) vanish and we have a reasonable approximation.
		We can then substitute this approximation into the inner maximization of Eq. (\ref{eq:trades}), which gives a Jacobian type penalty
		\begin{equation}\label{eq:second order trades}
			\minimize_{f \in \mc{F}} \, \mb{E}_{(\bm{x}, y) \sim \mc{D}} \left[ \mc{L}(\bm{e}_y, f(\bm{x})) + \frac{ \beta \epsilon^2 }{2} \dnorm{ \bm{F}_f^{1/2} \bm{J}(\bm{x}) }_{p \hookrightarrow 2}^2 \right].
		\end{equation}
		Note that here, $\norm{\cdot}_{p \hookrightarrow q}$ indicates the vector-induced $(p,q)$-norm (see Appendix \ref{appendix:sec:notation}).
		See Appendix \ref{appendix:derivation of second order trades} for the full derivation.
		This form is similar to the decomposition of the canonical adversarial training program Eq. (\ref{eq:canonical adversarial training}) into a gradient regularization term Eq. (\ref{eq:gradient penalty training}).
		In the case of Eq. (\ref{eq:second order trades}), however, we regularize the entire Jacobian matrix (through the FIM) rather than just one gradient direction corresponding to a particular label.
		
		The following lemma provides bounds on the FIM penalty in Eq. (\ref{eq:second order trades}) which are necessary for the proof of Proposition \ref{proposition:regional smoothness} below.
		\begin{lemma}\label{lemma:fisher information eigenvalue bounds}			
			Let $\mc{L}$ be the cross-entropy loss and assume that $f \st \mc{X} \to (0,1)^K$ is once-differentiable.
			If $p \geq 2$, then for each $\bm{x} \in \mc{X}$, the FIM $\bm{F}_{\bm{x}}$ satisfies
			\begin{equation}\label{eq:fim max eigenvalue bounds}
				\dnorm{ \bm{J}(\bm{x})^\top \bm{F}_f^{1/2} }_{2,\infty}^2 
					\leq \lambda_{\max}(\bm{F}_{\bm{x}}) 
					\leq \dnorm{ \bm{F}_f^{1/2} \bm{J}(\bm{x}) }_{p \hookrightarrow 2}^2
					\leq \dnorm{ \bm{J}(\bm{x})^\top \bm{F}_f^{1/2} }_F^2.
			\end{equation}
			Moreover,
			\begin{equation}\label{eq:jacobian penalty upper bound}
				\dnorm{ \bm{J}(\bm{x})^\top \bm{F}_f^{1/2} }_F \geq \dnorm{ \bm{J}(\bm{x}) }_F.
			\end{equation}
		\end{lemma}
		
		\noindent
		See Appendix \ref{appendix:sec:proof of proposition regional smoothness} for proof.
		It can be easily seen from Lemma \ref{lemma:fisher information eigenvalue bounds} Eq. (\ref{eq:fim max eigenvalue bounds}) that Eq. (\ref{eq:second order trades}) is upper bounded by the following problem:
		\begin{equation}\label{eq:frobenius norm regularizer}
			\minimize_{f \in \mc{F}} \, \mb{E}_{(\bm{x}, y) \sim \mc{D}} \left[ \mc{L}(\bm{e}_y, f(\bm{x})) + \frac{ \beta \epsilon^2 }{2} \dnorm{ \bm{F}_f^{1/2} \bm{J}(\bm{x}) }_F^2 \right],
		\end{equation}
		which we use in our experiments in Section \ref{sec:cifar10 recognizability experiments} below.
		
		Now that we have Lemma \ref{lemma:fisher information eigenvalue bounds}, it is worth briefly pointing out explicitly how the FIM penalty of Eq. (\ref{eq:second order trades}) is related to various robustness promoting regularizers suggested in prior literature.
		\begin{itemize}
			\item For an $\ell_2$ adversary, the Eq. (\ref{eq:second order trades}) penalty is the spectral norm $\norm{ \bm{F}_f^{1/2} \bm{J}(\bm{x}) }_{2 \hookrightarrow 2} = \lambda_{\max}(\bm{F}_{\bm{x}})$, which was considered in \citet{zhao2019adversarial, miyato2018virtual}.

			\item From the upper bound problem Eq. (\ref{eq:frobenius norm regularizer}), it can be seen Lemma \ref{lemma:fisher information eigenvalue bounds} Eq. (\ref{eq:jacobian penalty upper bound}) in the appendix) that the $\norm{ \bm{F}_f^{1/2} \bm{J}(\bm{x}) }_F$ term also upper bounds the Jacobian penalty $\norm{ \bm{J}(\bm{x}) }_F$, which was considered by \citet{hoffman2019robust, moosavi2019robustness, jakubovitz2018improving}.
		\end{itemize}
		
		To make the smoothing of Eq. (\ref{eq:second order trades}) even more explicit, the following Proposition \ref{proposition:regional smoothness} shows that the FIM penalty corresponds to the smoothness of the classifier $f$ via its local Lipschitz constant.
		
		\begin{proposition}\label{proposition:regional smoothness}
			Assume an $\ell_p$ threat model with $p \geq 2$, let $\mc{L}$ is the cross-entropy loss, and assume that the classifier $f \st \mc{X} \to (0,1)^K$ is once-differentiable.
			If $\norm{ \bm{F}_f^{1/2} \bm{J}(\bm{x}) }_{p \hookrightarrow 2} \leq \frac{L}{\sqrt{K}}$ over some $B \subseteq \mc{X}$, then 
			\begin{enumerate}[label=(\roman*)]	
				\item $f$ is $L$-Lipschitz on $B$ and, moreover, each component $f(\cdot)[k]$ is $L / \sqrt{K}$-Lipschitz on $B$.
				
				\item If $f(\cdot) = \text{Softmax}( g( \cdot ))$ for some $g \st \mc{X} \to \real^K$, then (i) holds for $g$ as well.
			\end{enumerate}
		\end{proposition}
		
		\noindent
		See Appendix \ref{appendix:sec:proof of proposition regional smoothness} for proof.
		Note that since Eq. (\ref{eq:frobenius norm regularizer}) is an upper bound on Eq. (\ref{eq:second order trades}), Proposition \ref{proposition:regional smoothness} holds for a bounded $\norm{ \bm{F}_f^{1/2} \bm{J}(\bm{x}) }_F \leq \frac{L}{\sqrt{K}}$ as well.
		
		In other words, in Eq. (\ref{eq:second order trades}), the loss term $\mc{L}(\bm{e}_y, f(\bm{x}))$ orients the decision boundary of $f$ while the FIM penalty biases towards $f$ that are smooth around the data points.
		We can tune the smoothness by increasing the weight $\epsilon^2 \beta$, which corresponds to larger localities around the data that should be smooth.
		This further informs our hypothesis from the previous section that the human-recognizability phenomenon requires two ingredients in order to manifest: orientation of the decision boundary and smoothness of the model itself.
		
		We now experimentally investigate the impact of smoothness regularization on the human-recognizability of adversarial perturbations.

\section{Experimental Results}\label{sec:cifar10 recognizability experiments}


Per our analysis in the previous section, we now present an experiment providing evidence that a the human-recognizability of adversarial perturbations occurs when orientation of the decision boundary are smoothness of the decision surface are combined.
Our investigation in the prior section shows that this gives a model that is both smooth and has an adversarially aligned decision boundary. 

To design our experiment, we group the defenses into three categories: minimax defenses (FGSM, Madry, and TRADES), approximate defenses (first order Eq. (\ref{eq:gradient penalty training}) and second order Eq. (\ref{eq:frobenius norm regularizer})), and undefended models (defensive distillation and standard). 
The minimax and approximate defenses makes use of both decision boundary orientation and decision surface smoothness in the sense of Proposition \ref{proposition:regional smoothness}.
We also train a model via TRADES with a high penalty weight so that smoothness entirely dominates orientation of the decision boundary.
The first undefended model, defensive distillation, smooths the decision surface but does not adversarially orient the decision boundary since the underlying teacher model is trained via the standard learning program Eq. (\ref{eq:undefended training}). 
The second undefended model, standard training, is neither smooth nor decision boundary oriented. 

We then compare both the robustness to and human-recognizability of adversarial examples for each of the three categories. 
We find that when models are both smoothed and their decision boundaries are adversarially oriented, the models exhibit robustness and yield human-recognizable adversarial perturbations. 
We find that when the models are simply smoothed, as in defensive distillation, standard training, or highly penalized TRADES, the models do not yield human-recognizable adversarial perturbations.

\subsection{Methods}

For all experiments, we use a pre-activation ResNet18 and train on CIFAR-10. 
We augment the training data using random horizontal flipping and random cropping, normalizing each image according to the sample mean and standard deviation of the training data. 

We initialize all models' weights using the same random seed and train for 15 epochs using a batch size of 128 and cyclic learning rate schedule with maximum rate $0.2$ and minimum learning rate $0$: the learning rate increases for 7 epochs and decreases for the remaining epochs. 
For the optimizer, we use SGD with momentum 0.9 and weight decay $2 \times 10^{-4}$. 

For defensive distillation, we train a teacher model using softmax temperature $\tau = 100$.
We then distill this model again with the same temperature value.

For training the minimax and approximate defenses, we use an $\ell_\infty$ adversary with $\epsilon = 8 / 255$ where applicable -- note that the gradient and FIM penalties Eq. (\ref{eq:gradient penalty training}) and Eq. (\ref{eq:second order trades}) use regularization weights involving $\epsilon$.
For the gradient penalty approximate defenses, we choose a weight of $\beta = 64$, and we choose the norm order $q = 1$ (recall that this is the dual norm to the adversary's $\ell_p$ constraint) so as to correspond with an $\ell_\infty$-adversary. 
We implement this using double backpropagation in PyTorch \citep{paszke2019pytorch}.
For the FIM penalized defense, we use the Frobenius norm upper bound Eq. (\ref{eq:frobenius norm regularizer}), which we compute using a minor adaptation of \citep[Algorithm 1]{hoffman2019robust} (see Appendix \ref{appendix:computing the frobenius norm fim penalty} for more details).
We set $\beta = 1$.
For the FGSM defense, we use the improvements suggested by \citep{wong2020fast}: random initialization of the perturbation $\bm{\delta}$ within the $\ell_p$ $\epsilon$-ball and early stopping.
For the Madry defense, we run each perturbation for 7 iterations with no random restarts.
For the TRADES defense, we train models with both $\beta = 6$ and $\beta = 10,000$, again running each perturbation for 7 iterations with no random restarts.


To evaluate the robustness of the defenses, we generate a robust accuracy profile against an $\ell_\infty$ adversary (Figure \ref{fig:cifar10 robust accuracy}) on the CIFAR-10 test-set. 
We use Madry PGD attacks Eq. (\ref{eq:madry pgd iterates}) with an $\ell_\infty$ adversary using a stepsize $\alpha = 2 / 255$ for 50 iterations initialized randomly within the $\epsilon$-ball; if an adversarial example is not found (i.e. the classifier's prediction is not changed), we restart the perturbation up to a maximum of 10 total restarts.
We sweep the perturbation budget from 0/255, which gives the baseline accuracy of each model, and a maximum budget of 50/255, which is much larger than the perturbation budget that was used during training. 

To generate the adversarial examples to evaluate human recognizability, we use the targeted variant of Madry PGD Eq. (\ref{eq:madry pgd iterates}) (i.e. we select a class and approximate Eq. (\ref{eq:targeted attack})), setting the perturbation budget to $\epsilon = \infty$ i.e. we do not project the iterates at all since we are not concerned with remaining imperceptible.
We use a stepsize of $2 / 255$ and run each perturbation for 2,048 iterations.

All training and experiments are implemented in PyTorch \citep{paszke2019pytorch} and run on a single Tesla V100 GPU.
The source code for all of our experiments is made available\footnote{
    \url{https://github.com/cmu-sei/smoothness-and-recognizability}
}.

\subsection{Results}

Figure \ref{fig:cifar10 robust accuracy} displays the robust accuracy profiles for the seven defenses we consider and replicates findings from prior work. 
Note that profiles can be grouped by the category of defense. 
The undefended models have the highest accuracy on an unperturbed test set, but the accuracy quickly falls to zero for small amounts of perturbation. 
The minimax defenses have a different, more robust pattern, where their unperturbed accuracy is lower, but the model's accuracy decreases slowly as the strength of the perturbation increases. 
For this model on these data, TRADES is the most robust, followed closely by PGD and FGSM. 
The approximate methods fall somewhere between the minimax and undefended models, except that the unperturbed accuracy of the approximate methods is the lowest of the three categories.  
Note that we do not include the TRADES $\beta = 10,000$ model in this robustness evaluation because it converged to a classifier that predicts every input as the class \textit{ship} (the multi-class analogue of an all-ones classifier), which is completely robust but uninteresting for this evaluation.

 \begin{figure}
 	\centering
 	\includegraphics[width=.65\textwidth]{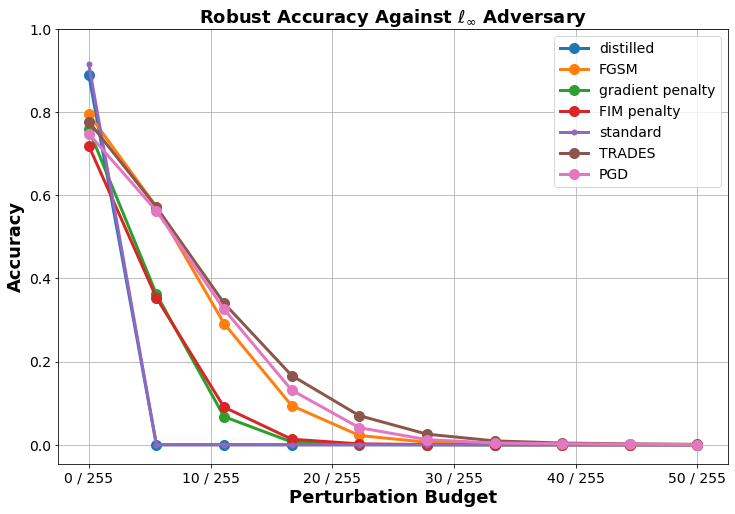}
 	\caption{
		Robust accuracy profiles against an $\ell_\infty$ adversary for the models trained in Section \ref{sec:cifar10 recognizability experiments}.
 	}
 	\label{fig:cifar10 robust accuracy}
 \end{figure}
	 
Figure \ref{fig:unconstrained adversarial perturbations} shows adversarial perturbations generated against the defended models, organized by the category of the defense. 
Part (a) gives the initial, random starts of all of the adversarial examples in a given class. 
Parts (b) and (c) give the adversarial examples generated against undefended models, a model with standard training and defensive distillation, respectively. 
Parts (d) and (e) display the results for the approximate methods, gradient penalization, and FIM penalization respectively. 
Parts (f)-(h) display the results for the minimax methods, for FGSM, Madry, and TRADES respectively. 
Part (i) shows the TRADES model trained with a high penalty $\beta = 10,000$ to be an all-\textit{ship} classifier.

		\begin{figure}
			\centering
			\subfloat[
			    The initial points to adversarial perturbations in (b)-(i).
			]{
			    \includegraphics[width=.8\textwidth]{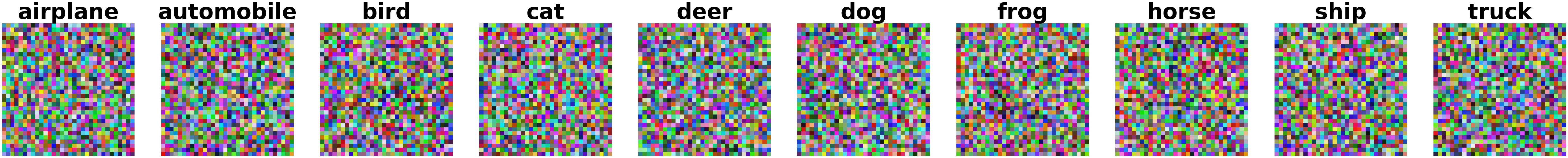}
			}
			\hfill 
			\subfloat[
				Standard model.			
			]{
				\includegraphics[width=.8\textwidth]{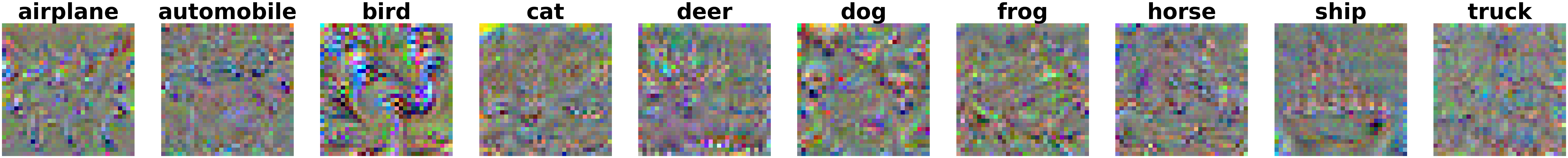}
			}
			\hfill
			\subfloat[
				Distilled model $\tau = 100$ -- distilled from the standard model in (a).
			]{
				\includegraphics[width=.8\textwidth]{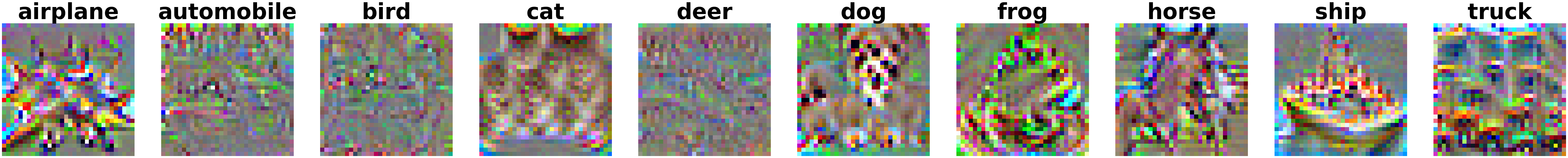}
			}
			\hfill 
			\subfloat[
				Gradient penalized model, $\beta = 64$.			
			]{
				\includegraphics[width=.8\textwidth]{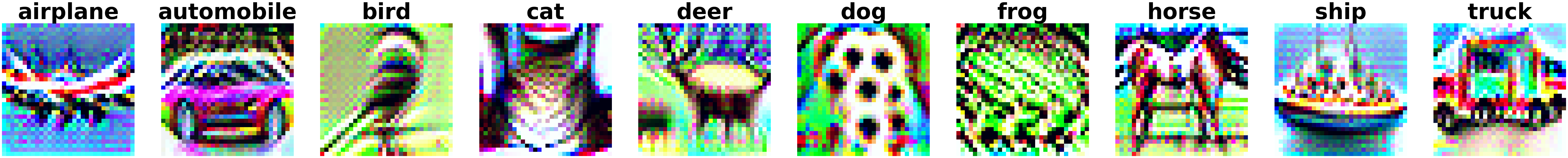}
			}
			\hfill
			\subfloat[
				FIM penalized model, $\beta = 1$.			
			]{
				\includegraphics[width=.8\textwidth]{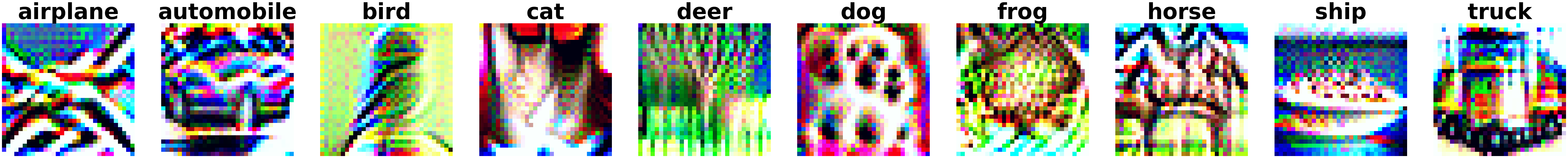}
			}
			\hfill
			\subfloat[
				FGSM $\ell_\infty$ model.			
			]{
				\includegraphics[width=.8\textwidth]{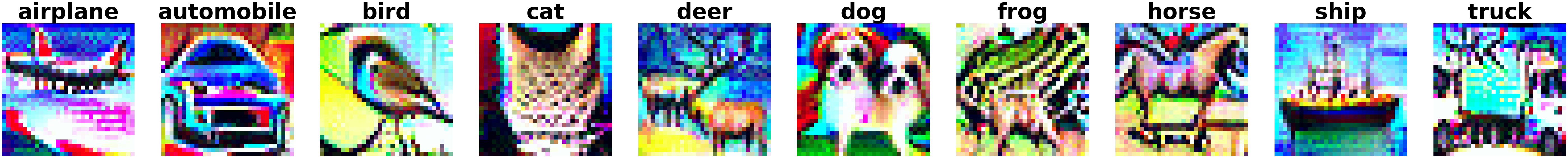}
			}
			\hfill
			\subfloat[
				Madry PGD $\ell_\infty$ model.			
			]{
				\includegraphics[width=.8\textwidth]{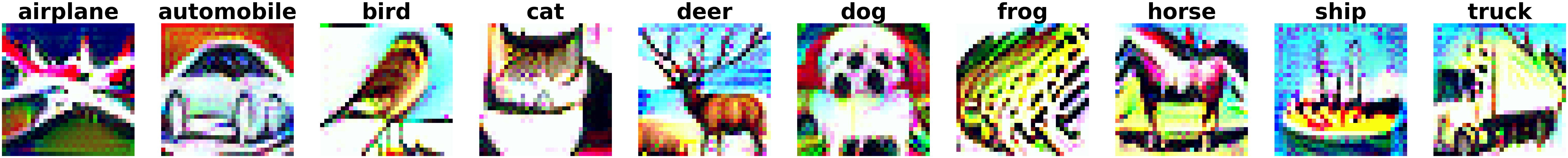}
			}
			\hfill
			\subfloat[
				TRADES $\ell_\infty$ model, $\beta = 6$.			
			]{
				\includegraphics[width=.8\textwidth]{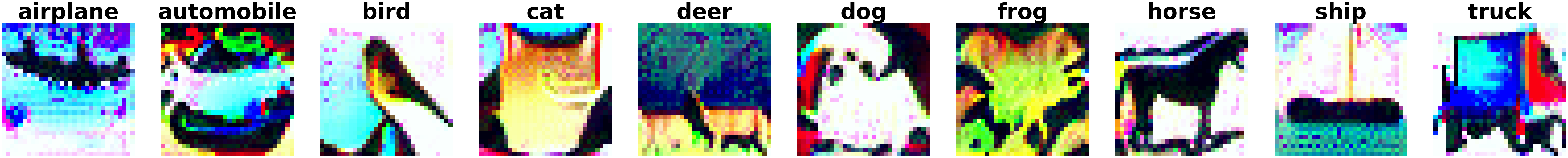}
			}
			\hfill
			\subfloat[
			    TRADES $\ell_\infty$ model, $\beta = 10,000$.
			]{
				\includegraphics[width=.8\textwidth]{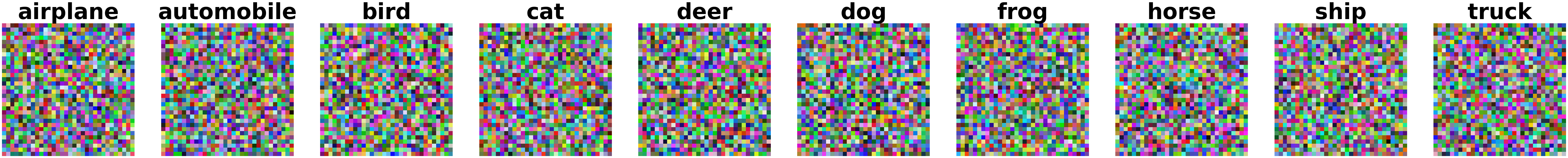}
			}
			\caption{
				Unconstrained adversarial perturbations for various pre-activation ResNet18 models.	
				The robust accuracy profiles of each model are shown in Figure \ref{fig:cifar10 robust accuracy}.
				Each image is the result of a distinct perturbation, which is initialized with random noise.
				In each column, we initialize the adversarial perturbations with the random noise shown in (a).
			}
			\label{fig:unconstrained adversarial perturbations}
		\end{figure}

Similarly to how the robustness accuracy profiles quantitatively distinguished between the three categories of defenses, the adversarial perturbations generated qualitatively distinguish between the three categories. 
First, the undefended models generated perturbations that are more perceptually similar to the initial point than they are to a human characterization of the target class. 
They are unrecognizable, although some of the standard model's perturbations are identifiable retrospectively given knowledge of the target class -- the bird and boat perturbations, for example. 
Similarly, the distilled model has various perceptually interesting patterns that appear, but those patterns are nonetheless largely insufficient to identify the class without knowledge of the target label.
It is worth pointing out, however, that the distilled model is distinctly more human-recognizable than the standard model.
For defensive distillation, the original teacher model's decision boundary orientation percolates to the distilled model while the high temperature smooths the decision surface, thus providing some notion of the smoothness and orientation coupling that we have considered throughout this work.

In contrast, the adversarial perturbations generated by the approximate defenses appear qualitatively different in kind from the starting images. 
These have patterns and colors that appear closer to exemplars of their classes relative to random noise, such that specific classes have recognizable features. 
For example, the gradient penalized method has recognizable exemplars of the various classes: automobile (tires and car shape), bird (pose, tail, and head), deer (shape, with antler-like structures), and ship (mast and water). 
Similarly, the FIM penalized method appears to have scattered versions of the gradient penalized examples.  

Similarly, the minimax defenses appear qualitatively different in kind from both the starting images and the adversarial perturbations generated from the approximate methods. 
The minimax defenses begin to show adversarial examples that are recognizable perturbations of the target classes, while the approximate methods merely had recognizable features. 
For example the FGSM dog, the Madry PGD deer, and the Madry PGD truck are particularly striking. 
Interestingly, the TRADES model seems to have less human-recognizable perturbations despite being the strongest smoothness regularizer of the three minimax defenses.
Since the approximate methods are also distinctly less human-recognizable than the FGSM and Madry PGD perturbations, it may be that decoupling the smoothness and orientation terms in the manner of TRADES Eq. (\ref{eq:trades}) requires more careful tuning of $\beta$ in order to match the same level of human-recognizability.

The highly penalized TRADES $\beta = 10,000$ model in Figure \ref{fig:unconstrained adversarial perturbations}(i) shows that if smoothness is practically speaking the only objective, then human-recognizability is lost.
This makes sense, since a classifier that only predicts one of the $K$ classes will yield adversarial gradient directions that are meaningless and, moreover, close to zero.

\section{Discussion}


Consider a situation in which a machine learning model is part of a high-stakes decision process that relies on sensitive training data. 
In cases like these, the high-stakes nature of the decision process suggests that the model needs to be robust to adversarial examples, but also the sensitivity of the data suggests that the model needs to also be protected from revealing certain types of information. 

In this paper, we have demonstrated that certain state of the art defenses against adversarial attacks lead to human recognizable features in adversarial examples generated from those defended models. 
Other works like \citet{kaur2019perceptually}, \citet{grathwohl2019your}, and \citet{yang2019me} have noticed the same phenomenon for other kinds of defenses, suggesting an incapability of the human-recognizability phenomenon.
We do not attempt to develop a privacy attack in this work, but we do claim that this kind of human-recognizability present in the adversarial examples suggests that privacy attacks are possible.
It is important for future research to pursue this line of inquiry to better understand the ways in which such privacy attacks may be possible and practically applicable.

For example, anecdotally, we have found that adversarial perturbations targeted towards the CIFAR-10 horse class often recover the presence of rider on the horses. 
Indeed, an inspection of the CIFAR-10 training set suggests that roughly 20\% of the exemplars of that class have riders present. 
If this information was not known \textit{a priori} to an attacker, they would be able to determine it if they had access to a robustly trained model's weights. 


\section{Conclusion}
In this work, we investigated the human-recognizability phenomenon of adversarially trained deep image classifiers that has been noted in a myriad of other works.
In particular, we investigated the human-recognizability of adversarial perturbations to models that were trained using both gradient and Jacobian-based regularization, which are fundamentally approximations of canonical adversarial training and TRADES adversarial training respectively.
We identified that state-of-the-art adversarial training approaches involve both a smoothness term and a term that orients the decision boundary, and that this coupling appears to be a sufficient condition for the human-recognizability of large-$\epsilon$ adversarial perturbations, which fundamentally rely on local information vis-\`a-vis gradient directions.
We demonstrated this sufficiency qualitatively using visualizations of the adversarial perturbations.
Finally, we discussed some implications of the human-recognizability phenomenon in the privacy setting, suggesting a direction for future research to consider.

\section{Acknowledgements}

Copyright 2020 Carnegie Mellon University. \\

\noindent This material is based upon work funded and supported by the Department of Defense under Contract No. FA8702-15-D-0002 with Carnegie Mellon University for the operation of the Software Engineering Institute, a federally funded research and development center. \\

\noindent The view, opinions, and/or findings contained in this material are those of the author(s) and should not be construed as an official Government position, policy, or decision, unless designated by other documentation. \\

\noindent NO WARRANTY. THIS CARNEGIE MELLON UNIVERSITY AND SOFTWARE ENGINEERING INSTITUTE MATERIAL IS FURNISHED ON AN "AS-IS" BASIS. CARNEGIE MELLON UNIVERSITY MAKES NO WARRANTIES OF ANY KIND, EITHER EXPRESSED OR IMPLIED, AS TO ANY MATTER INCLUDING, BUT NOT LIMITED TO, WARRANTY OF FITNESS FOR PURPOSE OR MERCHANTABILITY, EXCLUSIVITY, OR RESULTS OBTAINED FROM USE OF THE MATERIAL. CARNEGIE MELLON UNIVERSITY DOES NOT MAKE ANY WARRANTY OF ANY KIND WITH RESPECT TO FREEDOM FROM PATENT, TRADEMARK, OR COPYRIGHT INFRINGEMENT. \\

\noindent [DISTRIBUTION STATEMENT A] This material has been approved for public release and unlimited distribution.  Please see Copyright notice for non-US Government use and distribution. \\

\noindent Internal use:* Permission to reproduce this material and to prepare derivative works from this material for internal use is granted, provided the copyright and “No Warranty” statements are included with all reproductions and derivative works. \\

\noindent External use:* This material may be reproduced in its entirety, without modification, and freely distributed in written or electronic form without requesting formal permission. Permission is required for any other external and/or commercial use. Requests for permission should be directed to the Software Engineering Institute at permission@sei.cmu.edu. \\

\noindent * These restrictions do not apply to U.S. government entities. \\

\noindent DM20-0917

\bibliography{bibliography.bib}

\newpage
\appendix
\section{Notation}\label{appendix:sec:notation}
\begin{center}
	\def\arraystretch{1.75}
	\begin{tabular}{|r|l|}
		\hline
		\textbf{Notation} & \textbf{Description}
		\\\hline
		$\real$ & The set of real numbers.
		\\\hline 
		$\mc{X}$ & The input domain of a classification model.
		\\\hline
		$a, \ \bm{a}, \ \bm{A}$ & Scalar, vector, matrix respectively.
		\\\hline
		$\bm{a}[i]$ & The $i$-th entry of a vector $\bm{a}$.
		\\\hline
		$\bm{e}_i$ & The $i$-th canonical basis vector i.e. the $i$-th column of the identity \\ & matrix $\bm{I}$ i.e. a one-hot vector with $\bm{e}_i[i] = 1$.
		\\\hline 
		$\mb{E}$ & Expectation operator.
		\\\hline 
		$\bm{a} \circ \bm{b}$ & Hadamard product (entrywise multiplication) between $\bm{a}$ and $\bm{b}$.
		\\\hline
		$\norm{\bm{A}}_{p \hookrightarrow q} = \max_{\bm{x}} \ \norm{ \bm{A x} }_q / \norm{\bm{x}}_p$ & Vector-induced matrix norm.
		\\\hline
		$\norm{\bm{A}}_{p,q} = \dnorm{ \begin{bmatrix} \norm{ \bm{a}_1 }_q & \cdots & \norm{ \bm{a}_m }_q \end{bmatrix} }_p$ & Entrywise matrix norm i.e. $p$-norm of column-wise $q$-norm.
		\\\hline
		$\norm{\cdot}_F$ & Frobenius norm (equivalent to $\norm{\cdot}_{2,2}$).
		\\\hline
		$\text{Softmax} \st \real^K \to (0,1)^k$ & The softmax function defined entrywise as \\ & $\text{Softmax}(\bm{a})[k] \equiv \exp( a_k ) / \sum_i \exp(a_i)$.
		\\\hline
	\end{tabular}
\end{center}

\section{Derivation of Eq. (\ref{eq:second order trades})}\label{appendix:derivation of second order trades}
Let $f \st \mc{X} \to (0,1)^K$ be a classifier over $K$-many classes, where the domain is denoted by $\mc{X} \subseteq \real^d$.
Using a second-order Taylor expansion, it is straightforward to verify that
\begin{equation}\label{appendix:eq:second order taylor expansion}
\begin{aligned}
	\text{KL}( f(\bm{x}) \,||\, f(\bm{x} + \bm{\delta}) ) &= \mb{E}_y \left[ \log \frac{ f(\bm{x})[y] }{ f(\bm{x} + \bm{\delta})[y] } \right] 
		= \frac{1}{2} \bm{\delta}^\top \bm{F}_{\bm{x}} \bm{\delta} + \mc{O}( \dnorm{\bm{\delta}}_2^3 ),
\end{aligned}
\end{equation}
where $\bm{F}_{\bm{x}}$ is the FIM
\begin{equation}\label{appendix:eq:fim}
\begin{aligned}
	\bm{F}_{\bm{x}} &\equiv \mb{E}_y \left[ \mathlarger{\mathlarger{(}} \nabla_{\bm{x}} \mc{L}(\bm{e}_y, f(\bm{x})) \mathlarger{\mathlarger{)}} \, \mathlarger{\mathlarger{(}} \nabla_{\bm{x}} \mc{L}(\bm{e}_y, f(\bm{x})) \mathlarger{\mathlarger{)}}^\top \right]  \\[.5em]
		&= \sum_y f(\bm{x})[y] \, \mathlarger{\mathlarger{(}} \nabla_{\bm{x}} \mc{L}(\bm{e}_y, f(\bm{x})) \mathlarger{\mathlarger{)}} \, \mathlarger{\mathlarger{(}} \nabla_{\bm{x}} \mc{L}(\bm{e}_y, f(\bm{x})) \mathlarger{\mathlarger{)}}^\top \\[.5em]
		&= \sum_y f(\bm{x})[y] \, \left( \nabla_{\bm{x}} \log f(\bm{x})[y] \right) \left( \nabla_{\bm{x}} \log f(\bm{x})[y] \right)^\top \qquad (\text{for } \mc{L} \text{ cross-entropy}), \\[.5em]
\end{aligned}
\end{equation}
since the cross-entropy loss is $\mc{L}( \bm{e}_y, f(\bm{x}) ) = -\log f(\bm{x})[y]$.
By defining another FIM 
\begin{equation}\label{appendix:eq:fim 2}
\begin{aligned}
	\bm{F}_f &\equiv \mb{E}_y \left[ \mathlarger{\mathlarger{(}} \nabla_{f(\bm{x})} \mc{L}(\bm{e}_y, f(\bm{x})) \mathlarger{\mathlarger{)}} \, \mathlarger{\mathlarger{(}} \nabla_{f(\bm{x})} \mc{L}(\bm{e}_y, f(\bm{x})) \mathlarger{\mathlarger{)}}^\top \right] \\[.5em]
		&= \begin{bmatrix}
			f(\bm{x})[1] \\[.5em]
			&  \ddots \\[.5em]
			&& f(\bm{x})[K]
		\end{bmatrix}^{-1} \qquad (\text{for } \mc{L} \text{ cross-entropy}),
\end{aligned}
\end{equation}
we have by the chain rule that
$
	\bm{F}_{\bm{x}} = \bm{J}( \bm{x} )^\top \bm{F}_f \bm{J}( \bm{x} ),
$
where $\bm{J}(\bm{x})$ is the Jacobian of the model $f$ with respect to its input.
Note that $\bm{F}_f \succeq 0$ regardless of the choices of $\mc{L}$ and $f$, thus $\rank(\bm{F}_{\bm{x}}) \leq K$ and hence $\bm{F}_{\bm{x}}$ is singular when the data dimension $d > K$.
In fact, for the cross-entropy loss, we have that $\bm{F}_f \succ 0$.

Plugging the approximation Eq. (\ref{appendix:eq:second order taylor expansion}) into the inner maximization of the TRADES program Eq. (\ref{eq:trades}) yields
\begin{equation*}
\begin{aligned}
	\max_{\bm{\delta} \in B_\epsilon} \, \frac{1}{2} \bm{\delta}^\top \bm{F}_{\bm{x}} \bm{\delta} 
		&= \max_{ \norm{\bm{\delta}}_p^2 = 1 } \, \frac{\epsilon^2}{2} \bm{\delta}^\top \bm{F}_{\bm{x}} \bm{\delta} \\[.5em]
		&= \max_{ \norm{\bm{\delta}}_p^2 = 1 } \, \frac{\epsilon^2}{2} \bm{\delta}^\top  \bm{J}( \bm{x} )^\top \bm{F}_f^{1/2} \bm{F}_f^{1/2} \bm{J}( \bm{x} ) \bm{\delta} \\[.5em]
		&= \max_{\bm{\delta}} \, \frac{\epsilon^2}{2} \frac{ \dnorm{ \bm{F}_f^{1/2} \bm{J}(\bm{x}) \bm{\delta} }_2^2 }{ \norm{\bm{\delta}}_p^2 } \\[.5em]
		&= \frac{\epsilon^2}{2} \dnorm{ \bm{F}_f^{1/2} \bm{J}(\bm{x}) }_{p \hookrightarrow 2}^2,
\end{aligned}
\end{equation*}
since $\bm{F}_f \succ 0$ allows the square root $\bm{F}_{\bm{x}}^{1/2}$, whence TRADES is approximated as
\begin{equation}\label{appendix:eq:second order trades}
	\minimize_{f \in \mc{F}} \, \mb{E}_{(\bm{x}, y) \sim \mc{D}} \left[ \mc{L}(\bm{e}_y, f(\bm{x})) + \frac{\beta \epsilon^2}{2} \dnorm{ \bm{F}_f^{1/2} \bm{J}(\bm{x}) }_{p \hookrightarrow 2}^2 \right].
\end{equation}

\section{Proof of Proposition \ref{proposition:regional smoothness}}\label{appendix:sec:proof of proposition regional smoothness}

Recall the FIMs $\bm{F}_{\bm{x}}$ Eq. (\ref{appendix:eq:fim}) and $\bm{F}_f$ Eq. (\ref{appendix:eq:fim 2}), as well as the Jacobian of $f$ at $\bm{x}$, denoted $\bm{J}(\bm{x})$.
Now, to prove Proposition \ref{proposition:regional smoothness}, we will need two intermediary lemmas.
The first is Lemma \ref{appendix:lemma:fisher information eigenvalue bounds}, which provides some useful spectral bounds on the FIM $\bm{F}_{\bm{x}}$.

\begin{lemma}[Section \ref{sec:trades} Lemma \ref{lemma:fisher information eigenvalue bounds}]\label{appendix:lemma:fisher information eigenvalue bounds}
	Let $\mc{L}$ be the cross-entropy loss and assume that $f \st \mc{X} \to (0,1)^K$ is once-differentiable.
	For each $\bm{x} \in \mc{X}$, the FIM $\bm{F}_{\bm{x}}$ satisfies
	\begin{equation}\label{appendix:eq:fim max eigenvalue bounds}
		\dnorm{ \bm{J}(\bm{x})^\top \bm{F}_f^{1/2} }_{2,\infty}^2 \leq \lambda_{\max}(\bm{F}_{\bm{x}}) \leq \dnorm{ \bm{J}(\bm{x})^\top \bm{F}_f^{1/2} }_F^2.
	\end{equation}
	Moreover,
	\begin{equation}\label{appendix:eq:jacobian penalty upper bound}
		\dnorm{ \bm{J}(\bm{x})^\top \bm{F}_f^{1/2} }_F \geq \dnorm{ \bm{J}(\bm{x}) }_F.
	\end{equation}
\end{lemma}

\begin{proof}
	Denote the $k$-th row of $\bm{J}(\bm{x})$ as 
	\[
		\bm{g}_k \equiv -\nabla_{\bm{x}} \log f(\bm{x})[k] = -\frac{1}{ f(\bm{x})[k] } \nabla_{\bm{x}} f(\bm{x})[k],
	\] 
	noting that $f(\bm{x})[k] \in (0,1)$ for any $\bm{x} \in \mb{X}$, and observe subsequently that $\bm{F}_{\bm{x}} = \sum_k f(\bm{x})[k] \bm{g}_k \bm{g}_k^\top \succeq 0$.
	Then, by \citep[Theorem 1]{merikoski2004inequalities}, we have the lower bound
	\begin{align*}
		\lambda_{\max}(\bm{F}_{\bm{x}}) &\geq \max_k \frac{1}{ f(\bm{x})[k] } \ \lambda_{\max}\left( \bm{g}_k \bm{g}_k^\top \right) + \lambda_{\min}\left( \sum_{k \neq k^\star} f(\bm{x})[k] \ \bm{g}_k \bm{g}_k^\top \right) \\
			&\geq \max_k \frac{1}{ f(\bm{x})[k] } \dnorm{ \nabla_{\bm{x}} f(\bm{x})[k] }_2^2 \\[.5em]
			&= \dnorm{ \bm{J}(\bm{x})^\top \bm{F}_f^{1/2} }_{2,\infty}^2
	\end{align*}
	by definition of the entrywise $\norm{\cdot}_{2,\infty}$ norm, which is the $\ell_\infty$-norm of the column-wise $\ell_2$-norms.
	Note that $k^\star$ is the index on which the maximum is achieved and the second inequality follows since $\rank( \sum_{k \neq k^\star} \bm{g}_k \bm{g}_k^\top ) \leq K-1 < d$. 
	To complete Eq. (\ref{appendix:eq:fim max eigenvalue bounds}), observe that the upper bound is a simple application of the Cauchy-Schwarz inequality.
	
	To see Eq. (\ref{appendix:eq:jacobian penalty upper bound}), observe
	\begin{align*}
		\dnorm{ \bm{J}(\bm{x})^\top \bm{F}_f^{1/2} }_F^2 
			&= 
			\trace\left( \bm{F}_f \bm{J}(\bm{x}) \bm{J}(\bm{x})^\top \right) 
			= \sum_k \frac{1}{ f(\bm{x})[k] } \dnorm{ \nabla_{\bm{x}} f(\bm{x})[k] }_2^2 \\[.5em]
			&\geq \sum_k \dnorm{ \nabla_{\bm{x}} f(\bm{x})[k] }_2^2 
			= \dnorm{ \bm{J}(\bm{x}) }_F^2,
	\end{align*}
	where the inequality is because $f(\bm{x})[k] \in (0,1)$, thus completing the proof.
\end{proof}

\noindent
The second lemma inequality on the vector-induced operator norms of the kind found in Eq. (\ref{eq:second order trades}) for different choices of order $p$.
This will allow us to state our result for $\ell_p$ threat models with $p \geq 2$, which is reasonable since $p \geq 2$ are the most common $\ell_p$ threat models for adversarial training.

\begin{lemma}\label{appendix:lemma:induced matrix norm inequality}
	For $1 \leq p \leq q$, $r \geq 1$, and real matrix $\bm{A}$, 
	\[
		\dnorm{ \bm{A} }_{p \hookrightarrow r} \leq \dnorm{ \bm{A} }_{q \hookrightarrow r}.
	\]
\end{lemma}

\begin{proof}
	Let $\bm{x}^\star$ be the global maximizer of 
	\[
		\dnorm{ \bm{A} }_{p \hookrightarrow r} = \max_{\bm{x}} \, \frac{\dnorm{ \bm{A} \bm{x} }_r }{ \dnorm{ \bm{x} }_p }.
	\]
	Then 
	\begin{align*}
		\frac{\dnorm{ \bm{A} \bm{x}^\star }_r }{ \dnorm{ \bm{x}^\star }_p } \leq \frac{\dnorm{ \bm{A} \bm{x}^\star }_r }{ \dnorm{ \bm{x}^\star }_q }
			\leq \max_{\bm{x}} \frac{\dnorm{ \bm{A} \bm{x} }_r }{ \dnorm{ \bm{x} }_q }
			= \dnorm{ \bm{A} }_{q \hookrightarrow r}.
	\end{align*}
\end{proof}

\noindent
We are now ready to prove Proposition \ref{proposition:regional smoothness}.

\begin{proposition}[Proposition \ref{proposition:regional smoothness}]\label{appendix:proposition:regional smoothness}
	Assume an $\ell_p$ threat model with $p \geq 2$, let $\mc{L}$ is the cross-entropy loss, and assume that the classifier $f \st \mc{X} \to (0,1)^K$ is once-differentiable.
	If $\norm{ \bm{F}_f^{1/2} \bm{J}(\bm{x}) }_{p \hookrightarrow 2} \leq \frac{L}{\sqrt{K}}$ over some $B \subseteq \mc{X}$, then 
	\begin{enumerate}[label=(\roman*)]	
		\item $f$ is $L$-Lipschitz on $B$ and, moreover, each component $f(\cdot)[k]$ is $L / \sqrt{K}$-Lipschitz on $B$.
		
		\item If $f(\cdot) = \text{Softmax}( g( \cdot ))$ for some $g \st \mc{X} \to \real^K$, then (i) holds for $g$ as well.
	\end{enumerate}
\end{proposition}

\begin{proof}
	Recall that $f$, being once-differentiable, is $L$-Lipschitz on $B$ if and only if its Jacobian is bounded: $\norm{ \bm{J}(\bm{x}) }_2 \leq L$ for all $\bm{x} \in B$.
	We will show this condition by straightforward application of $\ell_p$-norm inequalities and Lemma \ref{appendix:lemma:fisher information eigenvalue bounds}.
	
	First, observe that 
	\[
	\begin{aligned}
		\dnorm{ \bm{J}(\bm{x})^\top \bm{F}_f^{1/2} }_{2,\infty}^2 &= \max_k \frac{1}{ f(\bm{x})[k] } \norm{ \nabla_{\bm{x}} f(\bm{x})[k] }_2^2 \\[.5em]
		&\geq \max_k \norm{ \nabla_{\bm{x}} f(\bm{x})[k] }_2^2 \\[.5em]
		&= \dnorm{ \bm{J}(\bm{x})^\top }_{2,\infty}^2,
	\end{aligned}
	\]
	since $f(\bm{x})[k] \in (0,1)$ for any $\bm{x} \in \mc{X}$. 
	It then follows that
	\begin{equation}\label{appendix:eq:proposition inequality 1}
	\begin{aligned}
		\dnorm{ \bm{J}(\bm{x})^\top \bm{F}_f^{1/2} }_{2,\infty}^2 &\geq \dnorm{ \bm{J}(\bm{x})^\top }_{2,\infty}^2 \\[.5em]
				&\geq \frac{1}{K} \dnorm{ \bm{J}(\bm{x}) }_F^2 \\[.5em]
				&\geq \frac{1}{K} \dnorm{ \bm{J}(\bm{x}) }_2^2.
	\end{aligned}
	\end{equation}
	Now, Lemma \ref{appendix:lemma:fisher information eigenvalue bounds} gives
	\begin{equation}\label{appendix:eq:proposition inequality 2}
	\begin{aligned}
		 \dnorm{ \bm{J}(\bm{x})^\top \bm{F}_f^{1/2} }_{2,\infty}^2 &\leq \lambda_{\max}(\bm{F}_{\bm{x}})
		 	= \lambda_{\max}\left( \bm{J}(\bm{x})^\top \bm{F}_f^{1/2} \bm{F}_f^{1/2} \bm{J}(\bm{x}) \right)
		 	\\[.5em]&= \dnorm{ \bm{F}_f^{1/2} \bm{J}(\bm{x}) }_{2 \hookrightarrow 2}^2 
		 	\\[.5em]&\leq \dnorm{ \bm{F}_f^{1/2} \bm{J}(\bm{x}) }_{p \hookrightarrow 2}^2 
			\\[.5em]&\leq \frac{L^2}{K},
	\end{aligned}
	\end{equation}
	where the second inequality is due to Lemma \ref{appendix:lemma:induced matrix norm inequality}.
	Combining (\ref{appendix:eq:proposition inequality 1}) and (\ref{appendix:eq:proposition inequality 2}) gives $\dnorm{ \bm{J}(\bm{x}) }_2 \leq L$, and it immediately follows that $f$ is $L$-Lipschitz on $B$ and thus each $f(\cdot)[k]$ is also $L/\sqrt{K}$-Lipschitz on $B$.
	This proves (i).
	
	Property (ii) follows immediately from (i) since $\text{Softmax}(\cdot)$ is a 1-Lipschitz function \citep[Proposition 4]{gao2017properties}.
\end{proof}

\section{Computing the Frobenius Norm FIM Penalty}\label{appendix:computing the frobenius norm fim penalty}
The Frobenius Norm FIM penalty in Eq. (\ref{eq:frobenius norm regularizer}) involves the term $\norm{ \bm{F}_f^{1/2} \bm{J}(\bm{x}) }_F^2$ that must be computed for each iteration of training.
The Jacobian, however, can be difficult to work with in practice for high dimensional datasets with many classes.
As such, an approximation algorithm is needed.
Such an algorithm is proposed in \citep[Algorithm 1]{hoffman2019robust} for computing $\norm{\bm{J}(\bm{x})}_F$, which is simple to adapt to our case involving the FIM by scaling the Jacobian-vector product by $\bm{F}_f^{1/2}$.


\section{Deriving The $\ell_p$ FGSM Algorithm}\label{appendix:deriving gradient penalty}
In this section, we derive a generic $\ell_p$-adversary version of the FGSM algorithm \citep{goodfellow2014explaining}, starting at the maximization problem that follows from a first-order Taylor expansion of the canonical adversarial learning program Eq. (\ref{eq:canonical adversarial training}).
Specifically, this maximization is
\[
	\max_{ \norm{\bm{\delta}}_p = \epsilon } \, \bm{\delta}^\top \nabla_{\bm{x}} \mc{L},
\]
where we use the shorthand $\nabla_{\bm{x}} \mc{L} = \nabla_{\bm{x}} \mc{L}(\bm{e}_y, f(\bm{x}))$.
The Lagrangian is $\mathfrak{L}(\bm{\delta}, \lambda) = \bm{\delta}^\top \nabla_{\bm{x}} \mc{L} - \lambda \left( \norm{\bm{x}}_p - \epsilon \right)$ and hence by the KKT conditions,
\[
	\nabla_{\bm{\delta}} \mathfrak{L} = \nabla_{\bm{x}} \mc{L} - \lambda \norm{\bm{\delta}}_p^{1-p} \varphi_{p-1}(\bm{\delta}) = \bm{0} \quad\Leftrightarrow\quad \nabla_{\bm{x}} \mc{L} = \lambda \norm{\bm{\delta}}_p^{1-p} \varphi_{p-1}(\bm{\delta}),
\]
where $\varphi_p(\bm{\delta}) \equiv \sign(\bm{\delta}) \circ |\bm{\delta}|^p$ with $\circ$ the Hadamard product and the absolute value and power taken entrywise.
Now, let $q \equiv p / (p - 1)$, noticing, then, that $(p-1) (q-1) = 1$ and thus $\varphi_{q-1}( \varphi_{p-1}( \bm{\delta} )) = \delta$.
It follows that
\[
	\varphi_{q-1}\left( \nabla_{\bm{x}} \mc{L} \right) = \left( \frac{ \lambda }{ \norm{ \bm{\delta} }_p^{p-1} } \right)^{q-1} \varphi_{q-1}( \varphi_{p-1}( \bm{\delta} )) = \frac{ \lambda^{q-1} }{ \norm{\bm{\delta}}_p } \bm{\delta}.
\]
Rearranging this equation, we have
\begin{equation*}
\begin{aligned}
	\frac{ \bm{\delta} }{ \norm{\bm{\delta}}_p } = \frac{1}{ \lambda^{q-1} } \varphi_{q-1}( \nabla_{\bm{x}} \mc{L} )
	\quad\Rightarrow\quad \lambda^{q-1} = \dnorm{ \varphi_{q-1}( \nabla_{\bm{x}} \mc{L} ) }_p
\end{aligned}
\end{equation*}
but, since feasibility dictates that $\norm{\bm{\delta}}_p = \epsilon$, we have 
\[
    \bm{\delta} = \frac{\epsilon}{\lambda^{q-1}} \varphi_{q-1}(\nabla_{\bm{x}} \mc{L}) = \epsilon \frac{ \varphi_{q-1}(  \nabla_{\bm{x}} \mc{L} ) }{ \dnorm{ \varphi_{q-1}( \nabla_{\bm{x}} \mc{L} ) }_p }.
\]
It is easy to see that in the $p = 2$ case, we obtain $\bm{\delta}^\star = \epsilon \norm{ \nabla_{\bm{x}} \mc{L} }_2^{-1} \nabla_{\bm{x}} \mc{L}$ and in the $p = \infty$ case, we have $\bm{\delta}^\star = \epsilon \sign(\nabla_{\bm{x}} \mc{L})$ -- precisely the FGSM update.
This means that we perturb the input as
\[
    \bm{x}' = \bm{x} + \epsilon \frac{ \varphi_{q-1}( \nabla_{\bm{x}} \mc{L} ) }{ \dnorm{ \varphi_{q-1}( \nabla_{\bm{x}} \mc{L} ) }_p }.
\]

\section{A Connection Between Model Inversion \& Adversarial Perturbation}\label{appendix:sec:model inversion}
In this section, we show that model inversion is equivalent to an unbounded adversarial perturbation in the sense that the algorithms for performing both types of attacks spawn from the same optimization problem.

The model inversion attack proposed in \citep{fredrikson2015model} can be written as a nonconvex program
\begin{equation}\label{appendix:eq:model inversion}
	\minimize_{\bm{x} \in C} \ 1 - f(\bm{x})[k] + \Omega(\bm{x}),
\end{equation}
where $C$ is a constraint set, $\Omega \st \mc{X} \to \real$ is a penalty function, and $f \st \mc{X} \to (0,1)^K$ is the classifier over $K$-many classes.
In \citep{fredrikson2015model}, $\Omega(\bm{x}) = 0$ and the constraint set is $C = \mc{X}$ or the codomain of a denoising autoencoder.
The solution to Eq. (\ref{appendix:eq:model inversion}) can be approximated via PGD.

We now show that Eq. (\ref{appendix:eq:model inversion}) as used by \citep{fredrikson2015model} can be thought of as an unbounded adversarial perturbation, which in turn provides an algorithmic connection between model inversion and adversarial perturbations.
Note that the solution to Eq. (\ref{appendix:eq:equivalent targeted adversarial attack}) below is a targeted adversarial attack Eq. (\ref{eq:targeted attack}).

\begin{proposition}\label{appendix:proposition:inversion is adversarial}
	Let $\Omega(\bm{x}) = 0$ for all $\bm{x}$.
	Then (\ref{appendix:eq:model inversion}) is equivalent to 
	\begin{equation}\label{appendix:eq:equivalent targeted adversarial attack}
	\begin{aligned}
		&& \minimize_{\bm{\delta} \in B} &\ L\left( \bm{e}_k, f(\bm{x} + \bm{\delta}) \right)
	\end{aligned}
	\end{equation}
	for a properly chosen $B$.
\end{proposition}
\begin{proof}
	Since $f(\bm{x})[k] \in [0,1]$, it is equivalent to solve 
	$$\maximize_{\bm{x} \in C} \ f(\bm{x})[k].$$
	Now, defining $L \st [0,1]^K \to \real$ as the cross-entropy loss $L(\bm{p}, \bm{q}) \equiv -\sum_k p_k \log q_k$, solving 
	$$\minimize_{\bm{x} \in C} \ L\left( \bm{e}_k, f(\bm{x}) \right),$$ 
	where $\bm{e}_k$ is a one-hot vector corresponding to the $k$-th class, is equivalent to solving (\ref{appendix:eq:model inversion}) since $L(\bm{e}_k, f(\bm{x})) = -\log f(\bm{x})[k]$ is a monotonically decreasing transformation of the previous objective, thereby converting local maxima to local minima and local minima to local maxima.
	Defining $B(\bm{x}) \equiv \set{ \bm{x}' - \bm{x} \st \bm{x}' \in C}$ as the feasible region yields the desired program (\ref{appendix:eq:equivalent targeted adversarial attack}).
\end{proof}

\end{document}